\documentclass{article}

\usepackage[final]{neurips_2023}
\usepackage{joeyzhou}
\usepackage{thm-restate}

\newtheorem{theorem}{Theorem}[section]

\newtheorem{claim}[theorem]{Claim}

\newtheorem{example}[theorem]{Example}
\newtheorem{corollary}[theorem]{Corollary}
\newtheorem{lemma}[theorem]{Lemma}

\usepackage[utf8]{inputenc} %
\usepackage[T1]{fontenc} %
\usepackage[hidelinks]{hyperref} %
\usepackage{url} %
\usepackage{booktabs} %
\usepackage{amsfonts} %
\usepackage{nicefrac} %
\usepackage{microtype} %
\usepackage{xcolor} %

\title{A Competitive Algorithm for Agnostic Active Learning}

\pdfoutput=1

\author{%
  Eric Price \\
  Department of Computer Science\\
  University of Texas at Austin\\
  \texttt{ecprice@cs.utexas.edu} \\
  \And
  Yihan Zhou \\
  Department of Computer Science\\
  University of Texas at Austin\\
  \texttt{joeyzhou@cs.utexas.edu} \\
}

\DeclareMathOperator*{\Xcal}{\mathcal{X}}

\newcommand{\Acal}{\mathcal{A}}

\DeclareMathOperator*{\err}{err}

\newcommand{\eps}{\varepsilon}
\newcommand{\wh}{\widehat}
\newcommand{\wt}{\widetilde}
\newcommand{\D}{\mathcal{D}}
\newcommand{\Ot}{\widetilde{O}}

\begin{document}

\maketitle

\begin{abstract}
  For some hypothesis classes and input distributions, \emph{active}
  agnostic learning needs exponentially fewer samples than passive
  learning; for other classes and distributions, it offers little to
  no improvement. The most popular algorithms for agnostic active
  learning express their performance in terms of a parameter called
  the disagreement coefficient, but it is known that these algorithms
  are inefficient on some inputs.

  We take a different approach to agnostic active learning, getting an
  algorithm that is \emph{competitive} with the optimal algorithm for
  any binary hypothesis class $H$ and distribution $\D_X$ over $X$.
  In particular, if any algorithm can use $m^*$ queries to get
  $O(\eta)$ error, then our algorithm uses $O(m^* \log |H|)$ queries to
  get $O(\eta)$ error. Our algorithm lies in the vein of the
  splitting-based approach of~\cite{Das04}, which gets a similar
  result for the realizable ($\eta = 0$) setting.

  We also show that it is NP-hard to do better than our algorithm's
  $O(\log |H|)$ overhead in general.
\end{abstract}

\section{Introduction}

Active learning is motivated by settings where unlabeled data is cheap
but labeling it is expensive. By carefully choosing which points to
label, one can often achieve significant reductions in label
complexity~\citep{CAL94}. A canonical example with exponential
improvement is one-dimensional threshold functions
$h_\tau(x):= 1_{x \geq \tau}$: in the noiseless setting, an active
learner can use binary search to find an $\eps$-approximation solution
in $O\bra{\log \frac{1}{\eps}}$ queries, while a passive learner needs
$\Theta\bra{\frac{1}{\eps}}$ samples~\citep{CAL94,dasgupta2005coarse,nowak2011geometry}.

In this paper we are concerned with agnostic binary classification.
We are given a hypothesis class $H$ of binary hypotheses
$h: \Xcal \to \{0, 1\}$ such that some $h^* \in H$ has
$\err(h^*) \leq \eta$, where the error
\[
  \err(h):= \Pr_{(x, y) \sim \D}[ h(x) \neq y]
\]
is measured with respect to an unknown distribution $\D$ over
$\Xcal \times \{0, 1\}$. In our \emph{active} setting, we also know
the marginal distribution $\D_X$ of $x$, and can query any point $x$
of our choosing to receive a sample $y \sim (Y \mid X=x)$ for
$(X, Y) \sim \D$. The goal is to output some $\wh{h}$ with
$\err(\wh{h}) \leq \eta + \eps$, using as few queries as possible.

The first interesting results for agnostic active learning were shown
by~\cite{BBL06}, who gave an algorithm called Agnostic Active (A\textsuperscript{2}) that gets
logarithmic dependence on $\eps$ in some natural settings: it needs
$\Ot\bra{\log \frac{1}{\eps}}$ samples for the 1d linear threshold
setting (binary search), as long as as $\eps > 16 \eta$, and
$\Ot\bra{d^2 \log \frac{1}{\eps}}$ samples for $d$-dimensional linear
thresholds when $\D_X$ is the uniform sphere and
$\eps > \sqrt{d} \eta$. This stands in contrast to the polynomial
dependence on $\eps$ necessary in the passive setting. The bound's
requirement that $\eps \gtrsim \eta$ is quite natural given a lower
bound of $\Omega\bra{d \frac{\eta^2}{\eps^2}}$ due
to~\citep{K06,BDL09}, where $d$ is the VC dimension. Subsequent works
have given new algorithms~\citep{dasgupta2007general,BHLZ10} and new
analyses \citep{hanneke2007bound} to get bounds for more general
problems, parameterized by the ``disagreement coefficient'' of the
problem. But while these can give better bounds in specific cases,
they do not give a good competitive ratio to the optimum algorithm:
see (\cite{hanneke2014theory}, Section 8.2.5) for a realizable example
where $O\bra{\log \frac{1}{\eps}}$ queries are possible, but
disagreement-coefficient based bounds lead to
$\Omega\bra{\frac{1}{\eps}}$ queries.

By contrast, in the \emph{realizable, identifiable} setting
($\eta = \eps = 0$), a simple greedy algorithm \emph{is} competitive
with the optimal algorithm. In particular,~\cite{Das04} shows that if
any algorithm can identify the true hypothesis in $m$ queries, then
the greedy algorithm that repeatedly queries the point that splits the
most hypotheses will identify the true hypothesis in $O(m \log |H|)$
queries. This extra factor of $\log |H|$ is computationally necessary:
as we will show in Theorem~\ref{thm:lower}, avoiding it is NP-hard in
general.
This approach can extend~\citep{dasgupta2005coarse} to the PAC setting
(so $\eps > 0$, but still $\eta = 0$), showing that if any algorithm
gets error $\eps$ in $m^*$ queries, then this algorithm gets error
$8 \eps$ in roughly $\Ot(m^* \cdot \log |H|)$ queries (but see the
discussion after Theorem 8.2 of~\cite{hanneke2014theory}, which points
out that one of the logarithmic factors is in an uncontrolled
parameter $\tau$, and states that ``Resolving the issue of this extra
factor of $\log \frac{1}{\tau}$ remains an important open problem in
the theory of active learning.'').

The natural question is: can we find an agnostic active learning
algorithm that is competitive with the optimal one in the agnostic
setting?

\paragraph{Our Results.} Our main result is just such a competitive
bound. We say an active agnostic learning algorithm $\mathcal{A}$
solves an instance $(H, \D_X, \eta, \eps, \delta)$ with $m$
measurements if, for every distribution $\D$ with marginal $\D_X$ and
for which some $h^* \in H$ has $\err(h^*) \leq \eta$, with probability
$1-\delta$, $\mathcal{A}$ uses at most $m$ queries and outputs
$\wh{h} \in H$ with $\err\bra{\wh{h}} \leq \eta + \eps$. Let
$m^*(H, \D_X, \eta, \eps, \delta)$ be the optimal number of queries
for this problem, i.e., the smallest $m$ for which any $\mathcal{A}$
can solve $(H, \D_X, \eta, \eps, \delta)$.

Define $N(H, \D_X, \alpha)$ to be the size of the smallest
$\alpha$-cover over $H$, i.e., the smallest set $S \subseteq H$ such
that for every $h \in H$ there exists $h' \in S$ with
$\Pr_{x \sim \D_X}[h(x) \neq h'(x)] \leq \alpha$. When the context is clear, we drop the parameters and simply use $N$. Of course, $N$ is
at most $\abs{H}$.

\begin{restatable}[Competitive Bound]{theorem}{mainthm}\label{thm:main}
  There exist some constants $c_1, c_2$ and $c_3$ such that for any instance $(H, \D_X, \eta, \eps, \delta)$ with
  $\eps \ge c_1\eta$, Algorithm~\ref{Alg:SAAAL} solves the instance with
  sample complexity
    \begin{align*}
    m(H, \D_X, \eta, \eps, \delta) &\lesssim \bra{m^*\bra{H, \D_X, c_2\eta, c_3\eps, \frac{1}{100}} + \log \frac{1}{\delta}} \cdot \log \frac{ N(H, \D_X, \eta)}{\delta}
  \end{align*}
  and polynomial time.
\end{restatable}

Even the case of $\eta = 0$ is interesting, given the discussion
in~\citep{hanneke2014theory} of the gap
in~\citep{dasgupta2005coarse}'s bound, but the main contribution is
the ability to handle the agnostic setting of $\eta > 0$. The
requirement that $\eps \ge O(\eta)$ is in line with prior
work~\citep{BBL06,dasgupta2005coarse}. Up to constants in $\eta$ and
$\eps$, Theorem~\ref{thm:main} shows that our algorithm is within a
$\log N \leq \log |H|$ factor of the optimal query complexity.

We show that it NP-hard to avoid this $\log N$ factor, even in the
realizable $(\eta = \eps = \delta = 0)$ case:

\begin{restatable}[Lower Bound]{theorem}{thmlower}\label{thm:lower}
  It is NP-hard to find a query strategy for every agnostic active
  learning instance within an $c\log|H|$ for some constant $c>0$ factor of the
  optimal sample complexity.
\end{restatable}

This is a relatively simple reduction from the hardness of
approximating \textsc{SetCover}~\citep{dinur2014analytical}. The lower
bound instance has $\eta = \eps = \delta = 0$, although these can be
relaxed to being small polynomials (e.g.,
$\eps = \eta = \frac{1}{3 \abs{X}}$ and
$\delta = \frac{1}{3 \abs{H}}$).

\paragraph{Extension.}
We give an improved bound for our algorithm in the case of noisy binary search (i.e., $H$ consists of 1d threshold functions). When $\eta = \Theta(\eps)$,
$N(H, \D_X, \eps) = \Theta(\frac{1}{\eps})$ and
  $m^*(\eta, \eps,.01) = O(\log \frac{1}{\eps})$. Thus
Theorem~\ref{thm:main} immediately gives a bound of
$O(\log^2\frac{1}{\eps \delta})$, which is
nontrivial but not ideal. (For $\eta \ll \eps$, the same bound holds
since the problem is strictly easier when $\eta$ is smaller.)
However, the bound in Theorem~\ref{thm:main} is quite loose in this
setting, and we can instead give a bound of
\[
  O\bra{\log \frac{1}{\eps\delta} \log \frac{\log
    \frac{1}{\eps}}{\delta}}
\]
for the same algorithm, Algorithm~\ref{Alg:SAAAL}. \redtext{This matches the
bound given by disagreement coefficient based algorithms for constant $\delta$.} The proof
of this improved dependence comes from bounding a new parameter
measuring the complexity of an $H, \D_x$ pair; this parameter is
always at least $\Omega(\frac{1}{m^*})$ but may be much larger (and is
constant for 1d threshold functions). See Theorem~\ref{thm:upperbeta}
for details.

\subsection{Related Work}

Active learning is a widely studied topic, taking many forms beyond
the directly related work on agnostic active learning discussed
above~\citep{settles.tr09}. Our algorithm can be viewed as
similar to ``uncertainty sampling''~\citep{lewis1995sequential,lewis1994heterogeneous}, a popular empirical approach to
active learning, though we need some
modifications to tolerate adversarial noise.

\redtext{One problem related to the one studied in this paper is noisy binary search, which corresponds to
active learning of 1d thresholds. This has been extensively studied
in the setting of \emph{i.i.d.} noise~\citep{BZ74,BH08,DLU21} as well as
monotonic queries~\citep{KK07}. Some work in this vein has extended
beyond binary search to (essentially) active binary
classification~\citep{nowak2008generalized,nowak2011geometry}. These
algorithms are all fairly similar to ours, in that they do
multiplicative weights/Bayesian updates, but they query the
\emph{single} maximally informative point. This is fine in the i.i.d.
noise setting, but in an agnostic setting the adversary can corrupt
that query. For this reason, our algorithm needs to find a \emph{set}
of high-information points to query.
}

\redtext{Another related problems is decision tree learning. The realizable, noiseless case $\eta=\eps=0$ of our problem can be reduced to learning a binary decision tree with minimal depth. \citet{hegedHus1995generalized} studied this problem and gave basically the same upper and lower bound as in \citet{dasgupta2005coarse}. \citet{kosaraju2002optimal} studied a split tree problem, which is a generalization of binary decision tree learning, and also gave similar bounds. \citet{Azad2022} is a monograph focusing on decision tree learning, in which many variations are studied, including learning with noise. However, this line of work usually allows different forms of queries so their results are not directly comparable from results in the active learning literature.}

\redtext{For much more work on the agnostic active binary classification problem, see \citet{hanneke2014theory} and references therein. Many of these papers give bounds in terms of the disagreement coefficient, but sometimes in terms of other parameters. For example, \citet{katz2021improved} has a query bound that is always competitive with the disagreement coefficient-based methods, and sometimes much better; still, it is not competitive with the optimum in all cases.}

\redtext{In terms of the lower bound, it is shown in \citet{laurent1976constructing} that the problem is NP-complete, in the realizable and noiseless setting. To the best of our knowledge, our Theorem~\ref{thm:lower} showing hardness of approximation to within a $O(\log|H|)$ factor is new. }

\paragraph{Minimax sample complexity bounds.}
\citet{hanneke2015minimax} and \citet{hanneke2007teaching} have also given ``minimax'' sample complexity bounds for their algorithms, also getting a sample complexity within $O(\log |H|)$ of optimal. However, these results are optimal with respect to the sample complexity for the worst-case distribution over $y$ \emph{and $x$}.
But the unlabeled data $x$ is given as input. So one should hope for a bound with respect to optimal for the \emph{actual} $x$ and only worst-case over $y$; this is our bound.

We give the following example to illustrate that our bound, and indeed our algorithm, can be much better.
\begin{example}
Define a hypothesis class of $N$ hypotheses $h_1,\cdots,h_N$
, and $\log N+N$ data points $x_1,\cdots,x_{\log N+N}$. For each hypothesis $h_j$, the labels of the first $N$ points express $j$ in unary and the labels of the last $\log N$ points express $j$ in binary. We set $\eta=\eps=0$ and consider the realizable case.
\end{example}
In the above example, the binary region is far more informative than the unary region, but disagreement coefficient-based algorithms just note that every point has disagreement. Our algorithm will query the binary encoding region and take $O(\log N)$ queries. Disagreement coefficient based algorithms, including those in \citet{hanneke2015minimax} and \citet{hanneke2007teaching}, will rely on essentially uniform sampling for the first $\Omega(N/\log N)$ queries.
These algorithms are ``minimax'' over $x$, in the sense that \emph{if you didn't see any $x$ from the binary region}, you would need almost as many samples as they use. But you \emph{do} see $x$ from the binary region, so the algorithm should make use of it to get exponential improvement.

\paragraph{Future Work.}
Our upper bound assumes full knowledge of $\D_X$ and the ability to
query arbitrary points $x$. Often in active learning, the algorithm
receives a large but not infinite set of unlabeled sample points $x$,
and can only query the labels of those points. How well our results
adapt to this setting we leave as an open question.

Similarly, our bound is polynomial in the number of hypotheses and the
domain size. This is hard to avoid in full generality---if you don't
evaluate most hypotheses on most data points, you might be missing the
most informative points---but perhaps it can be avoided in structured
examples.

\section{Algorithm Overview}

Our algorithm is based on a Bayesian/multiplicative weights type approach to the problem, and is along the lines of the splitting-based
approach of~\cite{Das04}.

We maintain a set of weights $w(h)$ for each $h \in H$, starting at
$1$; these induce a distribution
$\lambda(h):= \frac{w(h)}{\sum_h w(h)}$ which we can think of as our
posterior over the ``true'' $h^*$.

\paragraph{Realizable setting.}
As initial intuition, consider the realizable case of
$\eta = \eps = 0$ where we want to find the true $h^*$. If $h^*$
really were drawn from our prior $\lambda$, and we query a point $x$,
we will see a $1$ with probability $\E_{h \sim \lambda} h(x)$. Then
the most informative point to query is the one we are least confident
in, i.e., the point $x^*$ maximizing
\[
  r(x):= \min\curly{ \E_{h\sim \lambda}[h(x)], 1 - \E_{h\sim \lambda}[h(x)]}.
\]
Suppose an algorithm queries $x_1, \dotsc, x_m$ and receives \redtext{the majority label} under $h \sim \lambda$ each time. Then the fraction of
$h \sim \lambda$ that agree with \emph{all} the queries is at least
$1 - \sum_{i =1}^m r(x_i) \geq 1 - m r(x^*)$. This suggests that, if
$r(x^*) \ll \frac{1}{m}$, it will be hard to uniquely identify $h^*$.
It is not hard to formalize this, showing that: if no single
hypothesis has $75\%$ probability under $\lambda$, and any algorithm
exists with sample complexity $m$ and $90\%$ success probability at
finding $h^*$, we must have $r(x^*) \geq \frac{1}{10 m}$.

This immediately gives an algorithm for the $\eta = \eps = 0$ setting:
query the point $x$ maximizing $r(x)$, set $w(h) = 0$ for all
hypotheses $h$ that disagree, and repeat. As long as at least two
hypotheses remain, the maximum probability will be
$50\% < 90\%$ and each iteration will remove an
$\Omega(\frac{1}{m})$ fraction of the remaining hypotheses; thus after
$O(m \log H)$ rounds, only $h^*$ will remain.
This is the basis for~\cite{Das04}.

\paragraph{Handling noise: initial attempt.}
There are two obvious problems with the above algorithm in the
agnostic setting, where a (possibly adversarial) $\eta$ fraction of
locations $x$ will not match $h^*$. First, a single error will cause
the algorithm to forever reject the true hypothesis; and second, the
algorithm makes deterministic queries, which means adversarial noise
could be placed precisely on the locations queried to make the
algorithm learn nothing.

To fix the first problem, we can adjust the algorithm to perform
multiplicative weights: if in round $i$ we query a point $x_i$ and see
$y_i$, we set
\[
  w_{i+1}(h) =
  \begin{cases}
    w_i(h) & \text{if } h(x_i) = y_i\\
    e^{-\alpha} w_i(h) & \text{if } h(x_i) \neq y_i
  \end{cases}
\]
for a small constant $\alpha = \frac{1}{100}$. To fix the second
problem, we don't query the single $x^*$ of maximum $r(x^*)$, but
instead choose $x$ according to distribution $q$ over many points $x$ with
large $r(x)$.

To understand this algorithm, consider how $\log \lambda_i(h^*)$
evolves in expectation in each step. This increases if the query is correct, and
decreases if it has an error. A correct query increases $\lambda_i$
in proportion to the fraction of $\lambda$ placed on
hypotheses that get the query wrong, which is at least $r(x)$; and the probability of an error
is at most $\eta \max_x \frac{q(x)}{\D_x(x)}$. If at iteration $i$ the algorithm uses query distribution $q$, some calculation gives
that
\begin{align}
  \E_{q}\squr{\log \lambda_{i+1}(h^*) - \log \lambda_{i}(h^*)} \geq 0.9\alpha\left(\E_{x \sim q}[r(x)] - 2.3 \eta \max_x \frac{q(x)}{\D_x(x)}\right).
  \label{eq:potential}
\end{align}
The algorithm can choose $q$ to maximize this bound on the potential gain. There's a tradeoff between concentrating the samples over the $x$ of largest $r(x)$, and spreading out the samples so the adversary can't raise the error probability too high. We show that if learning is possible by any algorithm (for a constant factor larger $\eta$), then there exists a $q$ for which this potential gain is significant.

\begin{restatable}[Connection to OPT]{lemma}{lemrm}\label{lem:rm}
  Define $\norm{h - h'} = \Pr_{x \sim \D_x}[h(x) \neq h'(x)]$. Let
  $\lambda$ be a distribution over $H$ such that no
  radius-$(2\eta+\eps)$ ball $B$ centered on $h \in H$ has probability more than $80\%$. Let
    $m^* = m^*\bra{H, \D_X, \eta, \eps, \frac{1}{100}}$. Then there exists
  a query distribution $q$ over $\Xcal$ with
  \[
    \E_{x \sim q}[r(x)] - \frac{1}{10} \eta \max_x \frac{q(x)}{\D_X(x)} \geq \frac{9}{100 m^*}.
  \]
\end{restatable}

At a very high level, the proof is: imagine $h^* \sim \lambda$. If
the algorithm only sees the majority label $y$ on every query it
performs, then its output $\wh{h}$ is independent of $h^*$ and cannot
be valid for more than 80\% of inputs by the ball assumption; hence a
99\% successful algorithm must have a 19\% chance of seeing a minority
label. But for $m^*$ queries $x$ drawn with marginal distribution
$q$, without noise the expected number of minority labels seen is
$m^*\E[r(x)]$, so $\E[r(x)] \gtrsim 1/m^*$. With noise, the adversary
can corrupt the minority labels in $h^*$ back toward the majority,
leading to the given bound.

The query distribution optimizing~\eqref{eq:potential} has a simple structure: take a threshold $\tau$ for $r(x)$, sample from $\mathcal{D}_x$ conditioned on $r(x) > \tau$, and possibly sample $x$ with $r(x) = \tau$ at a lower rate. This means the algorithm can efficiently find the optimal $q$.

Except for the caveat about $\lambda$ not already concentrating in a
small ball, applying Lemma~\ref{lem:rm} combined
with~\eqref{eq:potential} shows that $\log \lambda(h^*)$ grows by
$\Omega(\frac{1}{m^*})$ in expectation for each query. It starts out
at $\log \lambda(h^*) = -\log H$, so after $O(m^* \log H)$ queries we
would have $\lambda(h^*)$ being a large constant in expectation (and
with high probability, by the modified Freedman's inequality (Theorem~\ref{thm:freedman})). Of course $\lambda(h^*)$ can't grow past $1$, which
features in this argument in that once $\lambda(h^*) > 80\%$, a small
ball \emph{will} have large probability and Lemma~\ref{lem:rm} no
longer applies, but at that point we can just output any hypothesis in
the heavy ball.

\begin{figure}
\centering
\begin{tabular}{c|c|c|}
& $\lambda(h)$ & Values $h(x)$\\
\hline
$h_1$ & 0.9 & 1111 1111\\
$h_2$ & $0.1 - 10^{-6}$ & 1111 0000\\
$h_3$ & $10^{-6}$ & 0000 1110\\
\hline
$y$ & & 0000 1111
\end{tabular}
\caption{An example demonstrating that the weight of the true hypothesis can decrease if $\lambda$ is concentrated on the wrong ball. In this example, the true labels $y$ are closest to $h_3$. But if the prior $\lambda$ on hypotheses puts far more weight on $h_1$ and $h_2$, the algorithm will query uniformly over where $h_1$ and $h_2$ disagree: the second half of points. Over this query distribution, $h_1$ is more correct than $h_3$, so the weight of $h_3$ can actually \emph{decrease} if $\lambda(h_1)$ is very large.}
\label{fig:hard}
\end{figure}

\paragraph{Handling noise: the challenge.} There is one omission in
the above argument that is surprisingly challenging to fix, and ends
up requiring significant changes to the algorithm: if at an
intermediate step $\lambda_i$ concentrates in the \emph{wrong} small
ball, the algorithm will not necessarily make progress. It is
entirely possible that $\lambda_i$ concentrates in a small ball, even
in the first iteration---perhaps $99\%$ of the hypotheses in $H$ are
close to each other. And if that happens, then we will have
$r(x) \leq 0.01$ for most $x$, which could make the RHS
of~\eqref{eq:potential} negative for all $q$.

In fact, it seems like a reasonable Bayesian-inspired algorithm really
must allow $\lambda(h^*)$ to decrease in some situations. Consider
the setting of Figure\redtext{~\ref{fig:hard}}. We have three hypotheses,
$h_1, h_2,$ and $h_3$, and a prior
$\lambda = (0.9, 0.099999, 10^{-6})$. Because $\lambda(h_3)$ is so
tiny, the algorithm presumably should ignore $h_3$ and query
essentially uniformly from the locations where $h_1$ and $h_2$
disagree. In this example, $h_3$ agrees with $h_1$ on all but an
$\eta$ mass in those locations, so even if
$h^* = h_3$, the query distribution can match $h_1$ perfectly and
not $h_3$. Then $w(h_1)$ stays constant while $w(h_3)$ shrinks.
$w(h_2)$ shrinks much faster, of course, but since the denominator is dominated by $w(h_1)$,
$\lambda(h_3)$ will still shrink. However, despite
$\lambda(h_3)$ shrinking, the algorithm is still making progress in
this example: $\lambda(h_2)$ is shrinking fast, and once it becomes
small relative to $\lambda(h_3)$ then the algorithm will start
querying points to distinguish $h_3$ from $h_1$, at which point
$\lambda(h_3)$ will start an inexorable rise.

Our solution is to ``cap'' the large density balls in $\lambda$,
dividing their probability by two, when applying Lemma~\ref{lem:rm}. Our algorithm maintains a set
$S \subseteq H$ of the ``high-density region,'' such that the
capped distribution:
\[
  \overline{\lambda}(h):=
  \begin{cases}
    \frac{1}{2}\lambda(h) & h \in S\\
    \lambda(h) \cdot \frac{1 - \frac{1}{2}\Pr[h \in S]}{1 - \Pr[h \in S]} & h \notin S
  \end{cases}
\]
has no large ball. Then Lemma~\ref{lem:rm} applies to
$\overline{\lambda}$, giving the existence of a query distribution $q$ so that
the corresponding $\overline{r}(x)$ is large. We then define the potential function
\begin{equation}\label{Equation:Potential}
  \phi_i(h^*):= \log \lambda_{i}(h^*) + \log \frac{\lambda_i(h^*)}{\sum_{h \notin S_i} \lambda_i(h)}
\end{equation}
for $h^* \notin S_i$, and $\phi_i = 0$ for $h^* \in S_i$. We show that
$\phi_i$ grows by $\Omega(\frac{1}{m^*})$ in expectation in each
iteration. Thus, as in the example of
Figure~\ref{fig:hard}, either $\lambda(h^*)$ grows as
a fraction of the whole distribution, or as a fraction of the
``low-density'' region.

If at any iteration we find that $\overline{\lambda}$ has some heavy
ball $B(\mu, 2\eta + \eps)$ so Lemma~\ref{lem:rm} would not apply, we
add $B\bra{\mu', 6\eta + 3\eps}$ to $S$, where
$B\bra{\mu',2\eta+\eps}$ is the heaviest ball under the capped distribution. We
show that this ensures that no small heavy ball exists in the capped
distribution $\overline{\lambda}$. Expanding $S$ only increases the
potential function, and then the lack of heavy ball implies the
potential will continue to grow.

Thus the potential~\eqref{Equation:Potential} starts at $-2\log |H|$, and grows by $\Omega(\frac{1}{m^*})$ in each iteration. After $O(m^* \log H)$ iterations, we will have $\phi_i \geq 0$ in
expectation (and with high probability by the modified Freedman's inequality (Theorem~\ref{thm:freedman})).
This is only possible if $h^* \in S$, which means that one of the
centers $\mu$ of the balls added to $S$ is a valid answer.

In fact, with some careful analysis we can show that with $1-\delta$
probability that one of the \emph{first} $O(\log \frac{H}{\delta})$ balls added to
$S$ is a valid answer. The algorithm can then check all the centers
of these balls, using the following active agnostic learning algorithm:

\begin{theorem}\label{thm:slowbound}
  Active agnostic learning can be solved for $\eps = 3 \eta$ with
  $O\bra{\abs{H} \log \frac{\abs{H}}{\delta}}$ samples.
\end{theorem}
\begin{proof}
  The algorithm is the following. Take any pair $h, h'$ with
    $\norm{h - h'} > 3\eta$. Sample $O\bra{\log \frac{\abs{H}}{\delta}}$
  observations randomly from $(x \sim \D_x \mid h(x) \neq h'(x))$.
  One of $h, h'$ is wrong on at least half the queries; remove it from
  $H$ and repeat. At the end, return any remaining $h$.

  To analyze this, let $h^* \in H$ be the hypothesis with error
  $\eta$. If $h^*$ is chosen in a round, the other $h'$ must have
  error at least $2 \eta$. Therefore the chance we remove $h^*$ is at
  most $\delta / \abs{H}$. In each round we remove a hypothesis, so
  there are at most $\abs{H}$ rounds and at most $\delta$ probability
  of ever crossing off $h^*$. If we never cross off $h^*$, at the end
  we output some $h$ with $\norm{h - h^*} \leq 3 \eta$, which gives
  $\eps = 3 \eta$.
\end{proof}

The linear dependence on $\abs{H}$ makes the
Theorem~\ref{thm:slowbound} algorithm quite bad in most
circumstances, but the dependence \emph{only} on $\abs{H}$ makes it
perfect for our second stage (where we have reduced to
$O(\log \abs{H})$ candidate hypotheses).

Overall, this argument gives an
$O\bra{m^* \log \frac{\abs{H}}{\delta} + \log \frac{\abs{H}}{\delta}\log \frac{\log
  \abs{H}}{\delta}}$ sample algorithm for agnostic active learning.
One can simplify this bound by observing that the set of centers $C$
added by our algorithm form a packing, and must therefore all be
distinguishable by the optimal algorithm, so $\log |C| \lesssim m^*$ when $|C|\ge2$.
This gives a bound of
\[
  O\bra{(m^* + \log \frac{1}{\delta}) \log \frac{\abs{H}}{\delta}}.
\]
By starting with an $\eta$-net of size $N$, we can reduce $\abs{H}$ to
$N$ with a constant factor increase in $\eta$.

With some properly chosen constants $c_4$ and $c_5$, the entire algorithm is formally described in Algorithm~\ref{Alg:SAAAL}, with the benchmark $m^*$ defined in Lemma~\ref{lem:Deltai}.
\paragraph{Remark 1:} As stated, the algorithm requires knowing $m^*$ to set the target sample complexity / number of rounds $k$. This restriction could be removed with the following idea. $m^*$ only enters the analysis through the fact that $\Omega\bra{\frac{1}{m^*}}$ is a lower bound on the expected increase of the potential function in each iteration. However, the algorithm \emph{knows} a bound on its expected increase in each round $i$; up to a constant factor, it is the value
\[
\tau_i=\max_q\E_{x \sim q}[\overline{r}_{i,S_i}(x)] - \frac{c_4}{20} \eta \max_x \frac{q(x)}{\D_X(x)}.
\]
optimized in the algorithm. Therefore, we could use an adaptive termination criterion that stops after query $k$ if $\sum_{i=1}^k\tau_i\ge c_6\log\frac{|H'|}{\delta}$, for a sufficiently large absolute constant $c_6$. By Lemma~\ref{lem:Deltai} and Theorem~\ref{thm:freedman}, this guarantees that $h^*$ has entered $S$ with high probability, so our analysis holds. Lemmas~\ref{Lemma:BallCapping} and~\ref{lem:rm} give $\tau_i\gtrsim1/m^*$ whenever $S_i\ne H'$, ensuring termination.

\paragraph{Remark 2:} The algorithm's running time is polynomial in $|H|$. This is in general not avoidable, since the input is a truth table for $H$. The bottleneck of the computation is the step where the algorithm checks if the heaviest ball has mass greater than 80\%. This step could be accelerated by randomly sampling hypothesis and points to estimate and find heavy balls; this would improve the dependence to nearly linear in $|H|$. If the hypothesis class has some structure, like the binary search example, the algorithm can be implemented more efficiently.

\begin{algorithm}
\caption{Competitive Algorithm for Active Agnostic Learning}\label{Alg:SAAAL}
\begin{algorithmic}
 \Statex \textbf{Stage 1: Generate candidate centres}
\State Compute a $2\eta$ maximal packing $H'$ with pairwise distances $>2\eta$.
 \State Let $w_1=1$ for every $h\in H'$; choose the absolute constant $\alpha=1/100$.
\State $S_0\gets \emptyset$
\State $C\gets ()$, an ordered list; $L\gets\left\lceil c_7\log(3|H'|/\delta)\right\rceil$ for a sufficiently large absolute constant $c_7$.
\State Set $R=c_4\eta+c_5\eps$ and $B(\mu,r)=\{h\in H':\norm{h-\mu}\le r\}$ for $\mu\in H$.
 \For { $i=1,\dots,k=1+\left\lceil c_8m^*\log\frac{3|H'|}{\delta}\right\rceil$ for a sufficiently large absolute constant $c_8$ }
\State Compute $\lambda_i(h)=\frac{w_i(h)}{\sum_{h\in H'}w_i(h)}$ for every $h\in H'$
 \If{there exists $\mu\in H$ such that $B(\mu,R)$ has probability $>80\%$ over $\overline{\lambda}_{i,S_{i-1}}$}
\State $S_i\gets S_{i-1}\cup B(\mu',3R)$ where $B(\mu',R)$ is the heaviest radius-$R$ ball over $\overline{\lambda}_{i,S_{i-1}}$, with $\mu'\in H$
 \State Append $\mu'$ to $C$
 \Else
\State $S_i\gets S_{i-1}$
\EndIf
 \If{$S_i=H'$}
 \State \textbf{break} and proceed to Stage 2
\EndIf
\State Compute $\overline{\lambda}_{i,S_i} =
  \begin{cases}
    \frac{1}{2}\lambda_i(h) & h \in S_i\\
    \lambda_i(h) \cdot \frac{1 - \frac{1}{2}\Pr_{h\sim\lambda_i}[h \in S_i]}{1 - \Pr_{h\sim\lambda_i}[h \in S_i]} & h \notin S_i
  \end{cases}$

\State Compute $\overline{r}_{i,S_i}(x)= \min\curly{ \E_{h\sim \overline{\lambda}_{i,S_i}}[h(x)], 1 - \E_{h\sim \overline{\lambda}_{i,S_i}}[h(x)]}$ for every $x\in\Xcal$
\State Find a query distribution by solving
  \begin{equation}\label{eq:optimization}
q^*=\argmax_q\left\{\E_{x \sim q}[\overline{r}_{i,S_i}(x)] - \frac{c_4}{20} \eta \max_x \frac{q(x)}{\D_X(x)}\right\}
\end{equation}
\State Query $x\sim q^*$, getting label $y$
\State Set %
\mbox{%
$w_{i+1}(h) =
  \begin{cases}
    w_i(h) & \text{if } h(x) = y\\
    e^{-\alpha} w_i(h) & \text{if } h(x) \neq y
  \end{cases}$
}%
 for every $h\in H'$
\EndFor
 \Statex \textbf{Stage 2: Select the final hypothesis}
\State Let $C_L$ be the first $\min\{|C|,L\}$ centres in $C$.
\State Find $\hat h$ in $C_L$ using Theorem~\ref{thm:slowbound} with error bound $(3+3c_4)\eta+3c_5\eps$ and failure probability $\delta/3$; if $C_L$ is empty, return any $h\in H'$.
 \State \Return $\hat{h}$
\end{algorithmic}
\end{algorithm}

\paragraph{Generalization for Better Bounds.} To get a better
dependence for 1d threshold functions, we separate out the
Lemma~\ref{lem:rm} bound on~\eqref{eq:potential} from the analysis of
the algorithm given a bound on~\eqref{eq:potential}. Then for
particular instances like 1d threshold functions, we get a better
bound on the algorithm by giving a larger bound
on~\eqref{eq:potential}.

\begin{restatable}{theorem}{thmupperbeta}\label{thm:upperbeta}
  Suppose that $\D_x$ and $H$ are such that, for any distribution
  $\lambda$ over $H$ such that no radius-$\bra{c_4\eta + c_5\eps}$ ball has
  probability more than $80\%$, there exists a distribution $q$ over
  $X$ such that
  \[
    \E_{x \sim q}[r(x)] - \frac{c_4}{20} \eta \max_x \frac{q(x)}{\D_x(x)} \geq \beta
  \]
  for some $\beta > 0$. Then for $\eps \ge c_1\eta$, $c_4\ge300$, $c_5=\frac{1}{15}$ and $c_1\ge90c_4$, let $N=N(H,\D_x,\eta)$
  be the size of an $\eta$-cover of $H$. Algorithm~\ref{Alg:SAAAL}, with $k=1+\lceil c_8\beta^{-1}\log(3|H'|/\delta)\rceil$,
   solves $(\eta, \eps, \delta)$ active agnostic learning with
  $O\bra{\frac{1}{\beta}\log \frac{N}{\delta} + \log \frac{N}{\delta} \log \frac{\log
    N}{\delta}}$ samples.
\end{restatable}

\begin{corollary}\label{corollary:binary}
  There exists a constant $c_1 > 1$ such that, for $1d$ threshold functions and $\eps > c_1\eta$,
  Algorithm~\ref{Alg:SAAAL} solves $(\eta, \eps, \delta)$ active agnostic
  learning with
  $O\bra{\log \frac{1}{\eps\delta} \log \frac{\log
    \frac{1}{\eps}}{\delta}}$ samples.
\end{corollary}
\begin{proof}
    Run Algorithm~\ref{Alg:SAAAL} with $\eta'=\eps/(2C)$ and
  $\eps'=\eps/2$, where $C>1$ is sufficiently large that
  Theorem~\ref{thm:upperbeta} applies, and assume $\eps>2C\eta$.
  Then $\eta'\ge\eta$, $\eps'=C\eta'$, and an output with error at most
  $\eta'+\eps'$ also has error at most $\eta+\eps$.
  The $1d$ threshold functions have an $\eta'$ -cover of size
    $N=O(1/\eps)$.
  To get the result by Theorem~\ref{thm:upperbeta}, it suffices to
  show $\beta = \Theta(1)$.

  Each hypothesis is of the form $h(x) = 1_{x \geq \tau}$, and
  corresponds to a threshold $\tau$. So we can consider $\lambda$ to
  be a distribution over $\tau$.

  Let $\lambda$ be any distribution for which no radius-$R$ with
  probability greater than $80\%$ ball exists, for
    $R = c_4 \eta' + c_5\eps'$. For any percent $p$ between $0$
  and $100$, let $\tau_p$ denote the pth percentile of $\tau$ under
  $\lambda$ (i.e., the smallest $t$ such that
  $\Pr[\tau \leq t] \geq p/100$). By the ball assumption,
    $ \tau_{9}$ and $\tau_{91}$ do not lie in the same radius-$R$ ball.
  Hence $\norm{h_{\tau_{9}} - h_{\tau_{91}}} > R$, or
  \[
    \Pr_x[ \tau_{9} \leq x < \tau_{91}] > R.
  \]
  We let $q$ denote $(\D_x \mid \tau_{9} \leq x < \tau_{91})$. Then
  for all $x \in \supp(q)$ we have $r(x) \geq 0.09$ and
  \[
    \frac{q(x)}{D_x(x)} = \frac{1}{\Pr_{x \sim D_x}[x \in \supp(q)]} < \frac{1}{R}.
  \]
  Therefore we can set
    \[
    \beta = \E_{x \sim q}[r(x)] - \frac{c_4}{20} \eta' \max_x \frac{q(x)}{D_x(x)} \geq 0.09 - \frac{c_4\eta'}{20(c_4\eta'+c_5\eps')} \ge0.04,
  \]
  as needed.
\end{proof}

\section{Proof of Lemma~\ref{lem:rm}}

\lemrm*
\begin{proof}
  WLOG, we assume that
  $\Pr_{h\sim\lambda}\squr{h(x)=0}\ge\Pr_{h\sim\lambda}\squr{h(x)=1}$
  for every $x\in\Xcal$. This means $r(x) = \E_{h\sim
    \lambda}[h(x)]$. This can be achieved by flipping all $h(x)$ and
  observations $y$ for all $x$ not satisfying this property; such an
  operation doesn't affect the lemma statement.

  We will consider an adversary defined by a function
  $g: X \to [0, 1]$. The adversary takes a hypothesis $h \in H$ and outputs a
  distribution over $y \in \{0, 1\}^X$ such that
  $0 \leq y(x) \leq h(x)$ always, and
  $\err(h) = \E_{x,y}[ h(x) - y(x)] \leq \eta$ always. For a
  hypothesis $h$, the adversary sets $y(x) = 0$ for all
  $x$ with $h(x) = 0$, and $y(x) = 0$ independently with probability
  $g(x)$ if $h(x) = 1$---unless $\E_x[h(x) g(x)] > \eta$, in which
  case the adversary instead simply outputs $y = h$ to ensure the
  expected error is at most $\eta$ always.

  We consider the agnostic learning instance where $x \sim \D_x$,
  $h \sim \lambda$, and $y$ is given by this adversary. Let $\Acal$
  be an $(\eta, \eps)$ algorithm which uses $m$ measurements and
  succeeds with $99\%$ probability. Then it must also succeed with
  $99\%$ probability over this distribution. For the algorithm to
  succeed on a sample $h$, its output $\wh{h}$ must have
  $\norm{h - \wh{h}} \leq 2 \eta + \eps$. By the bounded ball
  assumption, for any choice of adversary, no fixed output succeeds
  with more than $80\%$ probability over $h \sim \lambda$.

  Now, let $\Acal_0$ be the behavior of $\Acal$ if it observes $y=0$
  for all its queries, rather than the truth; $\Acal_0$ is independent
  of the input. $\Acal_0$ has some distribution over $m$ queries,
  outputs some distribution of answers $\wh{h}$. Let
    $q(x) = \frac{1}{m} \E[\text{number of queries to $x$ by $\Acal_0$}]$, so $q$ is a
  distribution over $\Xcal$. Since $\Acal_0$ outputs a fixed
  distribution, by the bounded ball assumption, for $h \sim \lambda$
  and arbitrary adversary function $g$,
  \[
    \Pr_{h \sim \lambda}[\Acal_0 \text{ succeeds}] \leq 80\%.
  \]
  But $\Acal$ behaves identically to $\Acal_0$ until it sees its first
  nonzero $y$. Thus,
  \[
    99\% \leq \Pr[\Acal \text{ succeeds}] \leq \Pr[\Acal_0 \text{ succeeds}] + \Pr[\Acal \text{ sees a non-zero } y]
  \]
  and so
  \[
    \Pr[\Acal \text{ sees a non-zero } y] \geq 19\%.
  \]
  Since $\Acal$ behaves like $\Acal_0$ until the first nonzero, we have
  \begin{align*}
    19\% &\leq \Pr[\Acal \text{ sees a non-zero } y]\\
         &= \Pr[\Acal_0 \text{ makes a query $x$ with $y(x) = 1$}]\\
         &\leq \E[\text{Number queries $x$ by } \Acal_0 \text{ with } y(x) = 1]\\
         &= m \E_{h \sim \lambda} \E_y \E_{x \sim q} [ y(x)].
  \end{align*}
  As an initial note, observe that $\E_{h,y}[y(x)] \leq \E_h[h(x)] = r(x)$ so
  \[
    \E_{x \sim q}[r(x)] \geq \frac{0.19}{m}.
  \]
  Thus the lemma statement holds for $\eta = 0$.

  \paragraph{Handling $\eta > 0$.}
  Consider the behavior when the adversary's function
  $g: X \to [0, 1]$ satisfies
  $\E_{x \sim \D_x}[g(x) r(x)] \leq \eta/10$. \redtext{We denote the class of all adversary satisfying this condition as $G$.} We have that
  \[
    \E_{h \sim \lambda}\squr{\E_{x \sim \D_x}[ g(x) h(x)] } = \E_{x \sim \D_x}[ g(x) r(x)] \leq \eta/10.
  \]
  Let $E_h$ denote the event that
  $\E_{x \sim \D_x}[ g(x) h(x)] \leq \eta$, so
  $\Pr[\overline{E}_h] \leq 10\%$. Furthermore, the adversary is
  designed such that under $E_h$, $\E_y[y(x)] = h(x)(1-g(x))$ for
  every $x$. Therefore:
  \begin{align*}
    0.19 &\leq \Pr[\Acal_0 \text{ makes a query $x$ with $y(x) = 1$}]\\
         &\leq \Pr[\overline{E}_h] + \Pr[\Acal_0 \text{ makes a query $x$ with $y(x) = 1$} \cap E_h]\\
         &\leq 0.1 + \E[\text{Number queries $x$ by } \Acal_0 \text{ with } y(x) = 1 \text{ and } E_h]\\
         &= 0.1 + m \E_h\squr{\mathbbm{1}_{E_h} \E_{x \sim q}[ \E_y y(x)]}\\
         &= 0.1 + m \E_h\squr{\mathbbm{1}_{E_h} \E_{x \sim q}[ h(x)(1-g(x))]}\\
         &\leq 0.1 + m \E_{x \sim q}[\E_h[h(x)](1-g(x))]\\
         &= 0.1 + m \E_{x \sim q}[r(x)(1-g(x))].
  \end{align*}
    Thus
  \begin{align}
    \max_q \min_{g\in G} \E_{x \sim q}[ r(x)(1-g(x))] \geq \frac{9}{100m}\label{eq:qg}
  \end{align}
  over all distributions $q$ and functions $g: X \to [0, 1]$
  satisfying $\E_{x \sim \D_x}[g(x) r(x)] \leq \eta/10$. We now try to
  understand the structure of the $q, g$ optimizing the LHS
  of~\eqref{eq:qg}.

  Let $g^*$ denote an optimizer of the objective. First, we show that the constraint is tight, i.e., $\E_{x \sim \D_x}[g^*(x) r(x)]=\eta/10$. Since increasing $g$ improves the constraint, the only way this could not happen is if the maximum possible function, $g(x) = 1$ for all $x$, lies in $G$. But for this function, the LHS of~\eqref{eq:qg} would be
  $0$, which is a contradiction; hence we know increasing $g$ to improve the objective at some point hits the constraint, and hence $\E_{x \sim \D_x}[g^*(x) r(x)]=\eta/10$.

  For any $q$, define $\tau_q \geq 0$ to be the minimum threshold such that
  \[
    \E_{x \sim \D_x}\squr{ r(x) \cdot 1_{\frac{q(x)}{\D_X(x)} > \tau_q}} < \eta/10.
  \]
  and define $g_q$ by
  \[
    g_q(x):=
    \begin{cases}
      1 & \frac{q(x)}{\D_X(x)} > \tau_q\\
      \alpha & \frac{q(x)}{\D_X(x)} = \tau_q\\
      0 & \frac{q(x)}{\D_X(x)} < \tau_q
    \end{cases}
  \]
  where $\alpha \in [0, 1]$ is chosen such that
  $\E_{x \sim \D_x}[r(x) g_q(x)] = \eta/10$; such an $\alpha$ always
  exists by the choice of $\tau_q$.

  For any $q$, we claim that the optimal $g^*$ in the LHS of~\eqref{eq:qg} is $g_q$. It needs to maximize
  \[
    \E_{x \sim \D_X}\squr{ \frac{q(x)}{\D_X(x)} r(x) g(x)}
  \]
  subject to a constraint on $\E_{x \sim \D_X}[ r(x) g(x)]$; therefore
  moving mass to points of larger $\frac{q(x)}{\D_X(x)}$ is always an
  improvement, and $g_q$ is optimal.

  We now claim that the $q$ maximizing~\eqref{eq:qg} has
  $\max \frac{q(x)}{\D_X(x)} = \tau_q$. If not, some $x'$ has $\frac{q(x')}{\D_X(x')} > \tau_q$.
  Then $g_q(x') = 1$, so the $x'$ entry contributes nothing to
  $\E_{x \sim q}[r(x)(1 - g_q(x))]$; thus decreasing $q(x)$ halfway
  towards $\tau_q$ (which wouldn't change $g_q$), and adding the savings
    in proportion to $\D_X(x)$ across all $q(x)$ (which shifts all densities $q(x)/\D_X(x)$ and the threshold by the same amount, so also does not change $g_q$) would
  increase the objective.

  So there exists a $q$ satisfying~\eqref{eq:qg} for which
  $\Pr\squr{\frac{q(x)}{\D_X(x)} > \tau_q} = 0$, and therefore the set
  $T = \curly{x \mid \frac{q(x)}{\D_X(x)} = \tau_q}$ satisfies
  $\E_{\D_X}\squr{r(x)\mathbbm{1}_{x \in T}}\geq\eta/10$ and a $g_q$ minimizing~\eqref{eq:qg} is
  \[
    g_q(x) = \frac{\eta}{10} \frac{\mathbbm{1}_{x \in T}}{\E_{\D_X}[ r(x)\mathbbm{1}_{x \in T}]}.
  \]
  Therefore
  \begin{align*}
    \E_{x \sim q}[ r(x) g_q(x)] &= \E_{x \sim \D_X}\squr{ \frac{q(x)}{\D_X(x)} r(x) \frac{\eta}{10} \frac{\mathbbm{1}_{x \in T}}{\E_{\D_X}[ r(x)\mathbbm{1}_{x \in T}]}}\\
                                &= \frac{\eta}{10} \max_x \frac{q(x)}{\D_X(x)}
  \end{align*}
  and so by~\eqref{eq:qg},
  \[
    \E_{x \sim q}[ r(x)] - \frac{\eta}{10} \max_x \frac{q(x)}{\D_X(x)} \geq \frac{9}{100 m}
  \]
  as desired.
\end{proof}

\section{Conclusion}
We have given an algorithm that solves agnostic active learning with (for constant $\delta$) at most an $O(\log |H|)$ factor more queries than the optimal algorithm. It is NP-hard to improve upon this $O(\log |H|)$ factor in general, but for specific cases it can be avoided. We have shown that 1d threshold functions, i.e. binary search with adversarial noise, is one such example where our algorithm matches the performance of disagreement coefficient-based methods and is within a $\log \log \frac{1}{\eps}$ factor of optimal.

\section{Acknowledgments}
Yihan Zhou and Eric Price were supported by
NSF awards CCF-2008868, CCF-1751040 (CAREER), and the NSF AI Institute for Foundations of Machine Learning (IFML).

\bibliographystyle{plainnat}
\bibliography{ref}

\appendix
\newpage

\section{Query Complexity Upper Bound}
In this section we present the whole proof of the query complexity upper bound of Algorithm \ref{Alg:SAAAL}, as stated in Theorem \ref{thm:main}.
\subsection{Notation}\label{Section:Notation}
We remind the readers about some definitions first. Remember that $w_i(h)$ denote the weight of hypothesis $h$ in iteration $i$ and $\lambda_{i,S}(h)=\frac{w_i(h)}{\sum_{h'\in S}w_i(h')}$ for some $S\subseteq H'$ denote the proportion of $h$ in $S$. We view $\lambda_{i,S}$ as a distribution of hypotheses in $S$ so for $h\notin S$, $\lambda_{i,S}(h)=0$. For a set $S \subseteq H'$ of hypotheses, we define
  $w_i(S):= \sum_{h \in S} w_i(h)$ and $\lambda_{i}(h) = \lambda_{i,H'}(h)$.

Define $r_{\lambda,h^*}(x):= \Pr_{h \sim \lambda}[h(x) \neq h^*(x)]$, and
$r_{\lambda}(x) = \min_{y \in \{0, 1\}} \Pr_{h \sim \lambda}[h(x) \neq y]$, so $r_{\lambda}(x) = \min(r_{\lambda,h^*}(x), 1-r_{\lambda,h^*}(x))$.

Define
  \begin{equation}
  \overline{\lambda}_{i,S}(h):= \frac{1}{2} \lambda_i(h) + \frac{1}{2} \lambda_{i, H' \setminus S}(h) =
  \begin{cases}
    \frac{1}{2}\lambda_i(h) & h \in S\\
    \lambda_i(h) \cdot \frac{1 - \frac{1}{2}\Pr_{h\sim\lambda_i}[h \in S]}{1 - \Pr_{h\sim\lambda_i}[h \in S]} & h \notin S
  \end{cases}
\end{equation}
as the "capped" distribution in iteration $i$.

Finally, for notational convenience define $r_{i,S}\coloneqq r_{\lambda_{i,S}}$, $r_{i,S,h}\coloneqq r_{\lambda_{i,S},h}$ and $\overline{r}_{i,S}\coloneqq r_{\overline{\lambda}_{i,S}}$.

The main focus of our proof would be analyzing the potential function
  \[
\phi_i(h^*)=
\begin{cases}
\log \lambda_{i}(h^*) + \log \lambda_{i,H'\setminus S_i}(h^*) & h^*\notin S_i\\
0 & h^*\in S_i,
\end{cases}
\]
where $h^*$ is the best hypothesis in $H'$. We use $w_i$ and $\lambda_i$ for the weights and distribution before query $i$, and $S_i$ for the capped set after the ball-addition step of iteration $i$. We would like to show that $\phi_{i+1}(h^*)-\phi_i(h^*)$ is growing at a proper rate in each iteration. We pick $S_i$ to be an expanding series of sets, i.e., $ S_i\subseteq S_{i+1}$ for any $i\ge1$. However, the change of the "capped" set $S_i$ makes this task challenging. Therefore, we instead analyze the following quantity defined as
  \[
    \Delta_i(h^*):=
    \begin{cases}
    \log \frac{\lambda_{i+1}(h^*)}{\lambda_{i}(h^*)} + \log \frac{\lambda_{i+1,H'\setminus S_i}(h^*)}{\lambda_{i,H' \setminus S_i}(h^*)} & h^*\notin S_i\\
    0 & h^*\in S_i,
    \end{cases}
  \]
and $\phi_{i+1}(h^*)-\phi_i(h^*)=\Delta_i(h^*)+\log\frac{\lambda_{i+1,H'\setminus S_{i+1}}(h^*)}{\lambda_{i+1,H'\setminus S_i}(h^*)}$ if $h^*\notin S_{i+1}$. Further, we define
\[
\psi_k(h^*):=\sum_{1\le i<k}\Delta_i(h^*),
\]

so by definition $\phi_k(h^*)= \phi_1(h^*)+ \psi_k(h^*) + \sum_{1\le i < k} \log \frac{\lambda_{i+1, H'\setminus S_{i+1}}(h^*)}{\lambda_{i+1, H' \setminus S_i}(h^*)}$ if $h^*\notin S_k$. In the following text, we will drop the parameter $h^*$ when the context is clear and just use $\phi_i$, $\Delta_i$ and $\psi_i$ instead. We use $\mathcal{F}_i$ to denote the $\sigma$-algebra of all information before query $i$, including $S_i$ and the preceding queried points and labels. Thus $\Delta_i$ is $\mathcal{F}_{i+1}$-measurable. Initially $\psi_1=0$ and $\phi_1\ge-2\log|H'|$. Once $h^*\in S_i$, all increments $\Delta_j$ for $j\ge i$ are zero, so $\psi$ and the error and ball counters defined below remain fixed from iteration $i$ onward. If the algorithm terminates, extend these processes, $S_i$, and $\phi_i$ constantly.

\subsection{Potential Growth}
We will lower bound the conditional per iteration potential increase by first introducing a lemma that relates the potential change to the optimization problem~(\ref{eq:optimization}).

\begin{lemma}\label{lem:lgrowth}
  Assume that $\err(h^*)\le\eta$, then for any set $S\subseteq H'$ of hypotheses containing $h^*$ and query distribution $q$, we have
    \[
    \E\squr{\log \frac{\lambda_{i+1,S}(h^*)}{\lambda_{i,S}(h^*)}\bigg|\mathcal{F}_i} \geq 0.9 \alpha\left(\E_{x \sim q}[r_{i,S,h^*}(x)] - 2.3 \eta \max_x \frac{q(x)}{D_X(x)}\right)
  \]
  for $\alpha \leq 0.2$. Moreover,
  \[
    \E\squr{\max\curly{0, \log \frac{\lambda_{i+1,S}(h^*)}{\lambda_{i,S}(h^*)}}\bigg|\mathcal{F}_i} \leq \alpha \E_{x \sim q}[r_{i,S,h^*}(x)].
  \]
\end{lemma}
\begin{proof}
  For notational convenience, define $\wt{r}(x):= r_{i,S,h^*}(x)$.

  Observe that
    \[
    \frac{\lambda_{i,S}(h^*)}{\lambda_{i+1,S}(h^*)} = \frac{w_{i}(h^*)}{w_{i+1}(h^*)}\frac{\sum_{h \in S} w_{i+1}(h)}{\sum_{h \in S} w_{i}(h)} = \frac{w_{i}(h^*)}{w_{i+1}(h^*)}\E_{h \sim \lambda_{i,S}}\squr{\frac{w_{i+1}(h)}{w_{i}(h)}}.
\]

Let
$p(x) = \Pr_{y \sim (Y \mid X)}[y \neq h^*(x)]$ denote the probability
of error if we query $x$, so
\[
  \E_{x \sim \D_X}[p(x)] \leq \eta.
\]
Suppose we query a point $x$ and do not get an error. Then
the hypotheses that disagree with $h^*$ are downweighted by an
$e^{-\alpha}$ factor, so
\[
  \frac{\lambda_{i,S}(h^*)}{\lambda_{i+1,S}(h^*)} = \E_{h \sim \lambda_{i,S}}[ 1 + (e^{-\alpha} - 1) 1_{h(x) \neq h^*(x)}] = 1 - (1-e^{-\alpha}) \wt{r}(x).
\]
On the other hand, if we do get an error then the disagreeing hypotheses are effectively upweighted by $e^{\alpha}$:
\[
  \frac{\lambda_{i,S}(h^*)}{\lambda_{i+1,S}(h^*)} = 1 + (e^{\alpha} - 1) \wt{r}(x).
\]
Therefore
\begin{align}
  &\E_{y \mid x}\squr{\log \frac{\lambda_{i+1,S}(h^*)}{\lambda_{i,S}(h^*)}\bigg|\mathcal{F}_i} \notag\\
  &= -(1 - p(x)) \log \left(1 - (1-e^{-\alpha}) \wt{r}(x)\right) - p(x) \log \left(1 + (e^{\alpha} - 1) \wt{r}(x)\right)\label{eq:yxlog}\\
  &\geq (1 - p(x)) (1-e^{-\alpha}) \wt{r}(x) - p(x) (e^{\alpha} - 1) \wt{r}(x)\notag\\
  &= (1 - e^{-\alpha})\wt{r}(x) - p(x)\wt{r}(x) (e^{\alpha} - e^{-\alpha}).\notag
\end{align}
Using that $\wt{r}(x) \leq 1$, we have
\begin{align*}
  \E\squr{\log \frac{\lambda_{i+1,S}(h^*)}{\lambda_{i,S}(h^*)}\bigg|\mathcal{F}_i} &\geq (1 - e^{-\alpha}) \E_{x \sim q}[\wt{r}(x)] - (e^{\alpha} - e^{-\alpha}) \E_{x \sim q}[p(x)]\\
  &\geq 0.9 \alpha \E_{x \sim q}[\wt{r}(x) - 2.3 p(x)],
\end{align*}
where the last step uses $\alpha \leq 0.2$. Finally,
\[
  \E_{x \sim q}[p(x)] = \E_{x \sim \D_X}\squr{ p(x) \frac{q(x)}{\D_X(x)} } \leq \eta \max_x \frac{q(x)}{\D_X(x)}.
\]
This proves the first desired result. For the second, note that if we query $x$, then conditioned on $\mathcal{F}_i$
  \[
  \max\curly{0, \log \frac{\lambda_{i+1,S}(h^*)}{\lambda_{i,S}(h^*)}} =
  \begin{cases}
    0 & \text{with probability } p(x),\\
    -\log (1 - (1 - e^{-\alpha}) \wt{r}(x)) & \text{otherwise.}
  \end{cases}
\]
Since $u\mapsto-\log(1-(1-e^{-\alpha})u)$ is convex on $[0,1]$, with endpoint values $0$ and $\alpha$, we have $-\log(1-(1-e^{-\alpha})\wt{r}(x))\le\alpha\wt{r}(x)$;
 taking the expectation over $x$ gives the result.
\end{proof}

The above lemma, combined with Lemma \ref{lem:rm}, proves the potential will grow at desired rate at each iteration. But remember that Lemma $\ref{lem:rm}$ requires the condition that no ball has probability greater than $80\%$, so we need to check this condition is satisfied. The following lemma shows that if we cap the set $S_i$, then the probability is not concentrated on any small balls.
\begin{lemma}\label{Lemma:BallCapping}
  In Algorithm~\ref{Alg:SAAAL}, for every iteration $i$ with $S_i\ne H'$, $S_i$ is such that no radius $c_4\eta+c_5\eps$ ball centered in $H$ has more than
  $80\%$ probability under $\overline{\lambda}_{i,S_i}$.
\end{lemma}
\begin{proof}
 Here all balls are intersected with $H'$. If $S_i=S_{i-1}$, then by the construction of $S_i$, there are no radius $c_4\eta+c_5\eps$ balls have probability greater than 80\% under $\overline{\lambda}_{i,S_{i-1}}=\overline{\lambda}_{i,S_i}$. Otherwise, we have $S_{i-1}\neq S_i$ and a ball $B(\mu,3c_4\eta+3c_5\eps)$ is added to $S_i$ in this iteration. We first prove a useful claim below.
\begin{claim}\label{Claim:ballweight}
If a ball $B'=(\mu,3c_4\eta+3c_5\eps)$ is added to $S_i$ at some iteration $i$, $\lambda_i(B(\mu,c_4\eta+c_5\eps))\ge0.6$.
\end{claim}
\begin{proof}
If $B'$ is added to $S_i$ at the iteration $i$, then the selected ball $D=B(\mu,c_4\eta+c_5\eps)$ satisfies $\bar{\lambda}_{i,S_{i-1}}(D)>0.8$. Since
\[
\bar{\lambda}_{i,S_{i-1}}(D)=\tfrac12\lambda_i(D)+\tfrac12\lambda_{i,H'\setminus S_{i-1}}(D)\le\tfrac12\lambda_i(D)+\tfrac12,
\]
we have $\lambda_i(D)>0.6$.
 \end{proof}

By Claim~\ref{Claim:ballweight}, the probability of $B(\mu,c_4\eta+c_5\eps)$ is at least $0.6$ over the uncapped distribution $\lambda_i$. So any ball not intersecting $B(\mu,c_4\eta+c_5\eps)$ has probability at most $0.4$ before capping. After capping these balls will have probability no more than $0.7$. At the same time, any ball intersects $B(\mu,c_4\eta+c_5\eps)$ would be completely inside $B(\mu,3c_4\eta+3c_5\eps)$ so its probability would be at most $0.5$ after capping.
\end{proof}
Now we are ready to apply Lemma \ref{lem:lgrowth} and Lemma \ref{lem:rm} except one caution. Remember that in the beginning of the algorithm, we compute a $2\eta$-packing $H'\subseteq H$ of the instance. From the well-known relationship between packing and covering (for example, see \citet[Lemma 4.2.8]{vershynin2018high}), we have $|H'|\le N(H,\eta)$. Every hypothesis in $H$ is within $2\eta$ to some hypothesis in $H'$, so there exists a hypothesis in $H'$ with error at most $3\eta$. This means that the best hypothesis $h^*\in H'$ has error $3\eta$ instead of $\eta$. The following lemma serves as the cornerstone of the proof of the query complexity upper bound, which states that the potential grows at rate $\Omega\bra{\frac{1}{m^*}}$ in each iteration.
\begin{lemma}\label{lem:Deltai}
Given $c_4\ge300$, $c_5\eps\ge(c_4+2)\eta$, $0<\alpha\le1/100$ and $\err(h^*)\le 3\eta$, the sampling distribution $q$ chosen by the algorithm satisfies
\[\begin{aligned}\E[\Delta_i|\mathcal{F}_i]&\ge\E\squr{\Delta_i|\mathcal{F}_i}-6\alpha \eta\max_x\frac{q(x)}{\D_X(x)}\\ &\gtrsim \frac{\alpha}{m^*\bra{H,\mathcal{D}_X,c_4\eta,c_5\eps-(c_4+2)\eta,\frac{1}{100}}}\quad\text{if}\quad h^*\notin S_i,\end{aligned}
  \]
  as well as $\abs{\Delta_i} \leq 2\alpha$ always and $\var[\Delta_i|\mathcal{F}_i] \le2\alpha\E\squr{|\Delta_i||\mathcal{F}_i}\le8\alpha\E[\Delta_i|\mathcal{F}_i]$.
\end{lemma}
\begin{proof}
  For the sake of bookkeeping, we let $m^*=m^*\bra{H,\D_X,c_4\eta,c_5\eps-(c_4+2)\eta,\frac{1}{100}}$ in this proof and the following text. If $h^*\in S_i$, then $\Delta_i=0$ and the increment and variance bounds are immediate. Otherwise, we first bound the expectation. By
  Lemma~\ref{lem:lgrowth} applied to $S \in \{H', H' \setminus S_i\}$ with $3\eta$, we have
    \begin{align*}
    \E\squr{\Delta_i|\mathcal{F}_i}-6\alpha \eta\max_x\frac{q(x)}{\D_X(x)} \geq &0.9 \alpha \left( \E_{x \sim q}[r_{i,H',h^*}(x) + r_{i,H' \setminus S_i,h^*}(x)] - 13.8 \eta \max_x \frac{q(x)}{\D_X(x)}\right)\\
    &-6\alpha \eta\max_x\frac{q(x)}{\D_X(x)},
  \end{align*}
  where $q$ is the query distribution of the algorithm at iteration $i$. Now, by the definition of
    \[
    \overline{\lambda}_{i, S} = \frac{1}{2} \lambda_{i} + \frac{1}{2}\lambda_{i, H' \setminus S},
  \]
  we have for any $x$ that
    \begin{align*}
    \overline{r}_{i,S_i,h^*}(x) &= \frac{1}{2} (r_{i,h^*}(x) + r_{i, H' \setminus S_i,h^*}(x))
  \end{align*}
  and thus
    \begin{align}
    \notag
    &\E\squr{\Delta_i|\mathcal{F}_i}-6\alpha \eta\max_x\frac{q(x)}{\D_X(x)}\\
    &\geq 1.8 \alpha \left( \E_{x \sim q}[\overline{r}_{i,S_i, h^*}(x)] - 6.9 \eta \max_x \frac{q(x)}{\D_X(x)}\right)\label{eq:Deltarh}-6\alpha \eta\max_x\frac{q(x)}{\D_X(x)}\\
    &\geq 1.8 \alpha \left( \E_{x \sim q}[\overline{r}_{i,S_i}(x)] - 10.3 \eta \max_x \frac{q(x)}{\D_X(x)}\right).\notag
  \end{align}
  Algorithm~\ref{Alg:SAAAL} chooses the
  sampling distribution $q$ to maximize $\E_{x \sim q}[\overline{r}_{i,S_i}(x)] - \frac{c_4}{20} \eta \max_x \frac{q(x)}{\D_X(x)}\le\E_{x \sim q}[\overline{r}_{i,S_i}(x)] - 15 \eta \max_x \frac{q(x)}{\D_X(x)}$ because $c_4\ge300$. By Lemma~\ref{Lemma:BallCapping}, $\overline{\lambda}_{i,S_i}$ over $H'$ has no radius-$(c_4\eta+c_5\eps)$ ball with probability larger than 80\%,
  so by Lemma~\ref{lem:rm}, applied with parameters $(c_4\eta,c_5\eps-c_4\eta)$, $q$ satisfies
    \begin{align*}
    \E_{x \sim q}[\overline{r}_{i,S_i}(x)] - 15\eta \max_x \frac{q(x)}{\D_X(x)}
    &\ge\E_{x \sim q}[\overline{r}_{i,S_i}(x)] - \frac{c_4}{20} \eta \max_x \frac{q(x)}{\D_X(x)}\\
    &\gtrsim\frac{1}{m^*\bra{H', \D_X, c_4\eta, c_5\eps-c_4\eta, \frac{1}{100}}}.
  \end{align*}
  Because $H'\subseteq H$ is a maximal $2\eta$-packing, every hypothesis in $H$ is within $2\eta$ of some hypothesis in $H'$. The problem $\bra{H,\D_X,c_4\eta,c_5\eps-(c_4+2)\eta,\frac{1}{100}}$ is harder than the problem $\bra{H',\D_X,c_4\eta,c_5\eps-c_4\eta,\frac{1}{100}}$ because we can reduce the latter to the former by simply adding more hypotheses and solve it then map the solution back by returning the closest hypothesis in $H'$. Hence, $m^*\ge m^*\bra{H',\D_X,c_4\eta,c_5\eps-c_4\eta,\frac{1}{100}}$. Therefore,
    \[
    \E\squr{\Delta_i|\mathcal{F}_i}-6\alpha \eta\max_x\frac{q(x)}{\D_X(x)} \geq 1.8 \alpha \bra{\E_{x \sim q}[\overline{r}_{i,S_i}(x)] - 10.3 \eta \max_x \frac{q(x)}{\D_X(x)}} \gtrsim \frac{\alpha}{m^*}.
  \]

  We now bound the variance. The value of $\Delta_i$ may be positive or negative,
  but it is bounded by $\abs{\Delta_i} \leq 2\alpha$: each of its two logarithms lies in $[-\alpha,\alpha]$. Thus
  \[
    \var[\Delta_i|\mathcal{F}_i] \leq \E\squr{\Delta_i^2|\mathcal{F}_i} \leq 2\alpha \E[\abs{\Delta_i}|\mathcal{F}_i].
  \]
    If $h^*\in S_i$, then $\Delta_i=0$ and the claimed bound is immediate. Otherwise, let
  \[
    a_i=\E_{x\sim q}[\overline{r}_{i,S_i,h^*}(x)],\qquad
    z_i=\eta\max_x\frac{q(x)}{\D_X(x)},\qquad
    d_i=\E[\Delta_i|\mathcal{F}_i].
  \]
    The positive value of the query objective and $c_4\ge300$ give $a_i\ge\E_{x\sim q}[\overline{r}_{i,S_i}(x)]\ge15z_i$. By~\eqref{eq:Deltarh},
  \[
    d_i\ge1.8\alpha(a_i-6.9z_i)\ge0.972\alpha a_i.
  \]
   Lemma~\ref{lem:lgrowth}, applied to both logarithms in $\Delta_i$, gives $\E[\max\{\Delta_i,0\}|\mathcal{F}_i]\le2\alpha a_i$. Therefore,
  \[
    \E[\abs{\Delta_i}|\mathcal{F}_i]
    =2\E[\max\{\Delta_i,0\}|\mathcal{F}_i]-d_i
    \le4\alpha a_i-d_i\le4d_i,
  \]
   and thus
    \[
    \var[\Delta_i|\mathcal{F}_i]\le2\alpha\E[\abs{\Delta_i}|\mathcal{F}_i]\le8\alpha\E[\Delta_i|\mathcal{F}_i].
  \]
\end{proof}

\subsection{Concentration of potential}
We have showed that the potential will grow at $\Omega\bra{\frac{1}{m^*}}$ per iteration while $h^*\notin S_i$, but only in expectation, while our goal is to obtain a high probability bound. Recall that $\psi_k=\sum_{1\le i<k}\Delta_i$. Let $\mu_k:= \sum_{1\le i<k}\E[\Delta_i\mid\mathcal{F}_i]$. Then, for $k\ge2$,
\begin{align*}
\E\squr{\bra{\psi_k-\mu_k}-\bra{\psi_{k-1}-\mu_{k-1}}|\mathcal{F}_{k-1}}&=\E\squr{\psi_k-\psi_{k-1}|\mathcal{F}_{k-1}}-\bra{\mu_k-\mu_{k-1}}\\
&=\E\squr{\Delta_{k-1}|\mathcal{F}_{k-1}}-\E\squr{\Delta_{k-1}|\mathcal{F}_{k-1}}=0.
\end{align*}
  Thus $\psi_k-\mu_k$ is a martingale adapted to $\mathcal{F}_k$. To show a high probability bound, we will use the following modification of Freedman's inequality in \citet{tropp2011freedman}. It allows the accumulated conditional variance to be bounded by a multiple of the random accumulated conditional mean.
\begin{theorem}[Modified Freedman's inequality] \label{thm:freedman}
  Let $\mathcal{F}_1\subseteq\mathcal{F}_2\subseteq\cdots$ be a filtration, and let $X_i$ be $\mathcal{F}_{i+1}$-measurable. Suppose that, almost surely,
\[
|X_i|\le b,\qquad d_i:=\E[X_i\mid\mathcal{F}_i]\ge0,\qquad
\var[X_i\mid\mathcal{F}_i]\le c_9 d_i
\]
  for constants $0<b<3/2$ and $c_9\ge0$. Define
\[
c:=1-\frac{c_9}{2(1-2b/3)}.
\]
Then, for every $\delta\in(0,1)$, with probability at least $1-\delta$, simultaneously for all $k\ge1$,
\[
\sum_{1\le i<k}X_i\ge c\sum_{1\le i<k}d_i-\log\frac{1}{\delta}.
\]
\end{theorem}

  \begin{proof}
Write $D_k=\sum_{1\le i<k}d_i$. With respect to the shifted filtration $\mathcal{G}_j:=\mathcal{F}_{j+1}$, consider the martingale
\[
Y_j=\sum_{i=1}^j(d_i-X_i),\qquad Y_0=0.
\]
Its differences $d_i-X_i$ have conditional mean zero and are at most $R:=2b$, and its predictable quadratic variation is
\[
W_j=\sum_{i=1}^j\var[X_i\mid\mathcal{F}_i]\le c_9 D_{j+1},\qquad W_0=0.
\]
By the scalar specializations of \citet[Lemmas 2.1 and 3.2]{tropp2011freedman}, applied to the differences divided by $R$ with exponential parameter $R$,
\[
Z_j:=\exp\bra{Y_j-\frac{e^R-1-R}{R^2}W_j}
\]
is a nonnegative supermartingale with $Z_0=1$. Since $R=2b<3$, the elementary inequality $e^u-1-u\le u^2/[2(1-u/3)]$ for $0\le u<3$ gives
\[
\frac{e^R-1-R}{R^2}\le\frac{1}{2(1-2b/3)}.
\]
The maximal inequality for nonnegative supermartingales gives $\Pr(\sup_{j\ge0}Z_j\ge1/\delta)\le\delta$. Thus, with probability at least $1-\delta$, simultaneously for all $j\ge0$,
\[
Y_j\le\frac{e^R-1-R}{R^2}W_j+\log\frac{1}{\delta}
\le\frac{c_9}{2(1-2b/3)}D_{j+1}+\log\frac{1}{\delta}.
\]
Taking $j=k-1$ and using $Y_{k-1}=D_k-\sum_{1\le i<k}X_i$, rearranging proves the result.
\end{proof}

Then we can prove a high probability bound as follows.
\begin{lemma}\label{Lemma:Concentration}
  With probability $1-\delta$, $\phi_i = 0$ for some $i=1+O\bra{m^* \log \frac{\abs{H'}}{\delta}}$ so $h^*\in S_i$.
\end{lemma}
\begin{proof}
  If $h^*\in S_1$, the conclusion holds immediately. Otherwise, whenever $h^*\notin S_k$, we have
\[
\phi_k=\phi_1+\psi_k+
\sum_{1\le i<k}\log\frac{\lambda_{i+1,H'\setminus S_{i+1}}(h^*)}{\lambda_{i+1,H'\setminus S_i}(h^*)}
\ge\phi_1+\psi_k,
\]
  because the sets $S_i$ are increasing. The uniform initial weights also give $\phi_1\ge-2\log|H'|$.

By Lemma~\ref{lem:Deltai}, $|\Delta_i|\le2\alpha$, the conditional means $\E[\Delta_i\mid\mathcal{F}_i]$ are nonnegative, and
\[
\var[\Delta_i\mid\mathcal{F}_i]\le8\alpha\E[\Delta_i\mid\mathcal{F}_i].
\]
  Apply Theorem~\ref{thm:freedman} with $X_i=\Delta_i$, $b=2\alpha$, and $c_9=8\alpha$. Taking $\alpha\le1/100$ ensures $c=1-4\alpha/(1-4\alpha/3)\ge1/2$. Thus, with probability at least $1-\delta$, simultaneously for all $k\ge1$,
\[
\psi_k\ge\frac{1}{2}\mu_k-\log\frac{1}{\delta}.
\]

Work on this event and suppose, for a contradiction, that $h^*\notin S_K$. Then $h^*\notin S_i$ for every $i\le K$, so Lemma~\ref{lem:Deltai} gives
\[
\mu_K\gtrsim\frac{\alpha(K-1)}{m^*}.
\]
For a fixed positive absolute constant $\alpha\le1/100$, choose $K=O\bra{m^*\log\frac{|H'|}{\delta}}$ with a sufficiently large implied constant that
\[
\frac{1}{2}\mu_K-\log\frac{1}{\delta}>2\log|H'|.
\]
Then $\phi_K\ge\phi_1+\psi_K>0$, contradicting $\phi_K\le0$. Therefore $h^*\in S_K$, and by definition $\phi_K=0$, with probability at least $1-\delta$. This proves the stated bound.
 \end{proof}
 \subsection{Bounding the Number of Candidates }
So far we've shown that after $O\bra{m^*\log\frac{|H'|}{\delta}}$ iterations, $h^*$ will be included in the set $S_i$. We next show that the ball that first contains $h^*$ is among the first $O\bra{\log\frac{|H'|}{\delta}}$ balls added. This suffices because stage two uses only the first $L$ centres in the ordered list $C$.
Let $e_i$ be the indicator that $h^*\notin S_i$ and $h^*$ errs on query $i$, and let $\mathcal{E}_k=\sum_{1\le i<k}e_i$. Let $\mathcal{N}_k$ count the balls added at iterations $i\le k$ for which $h^*\notin S_i$. Thus the ball that first puts $h^*$ in $S_i$, and every later ball, is excluded. Once $h^*\in S_k$, that first ball has index $\mathcal{N}_k+1$.
 \begin{lemma}\label{Lemma:GoodhExist}
The following inequality holds for every $k$:
\[
\mathcal{N}_k\le5(\psi_k+2\alpha \mathcal{E}_k)+1.
\]
\end{lemma}
\begin{proof}
 Write $R=c_4\eta+c_5\eps$. We divide the iterations before $h^*$ first enters $S_i$ into phases beginning at ball additions, with starting iterations $i_1<i_2<\cdots$. Let $B'_j=B(\mu_j,R)$ be the inner ball selected at iteration $i_j$. These inner balls are pairwise disjoint: if $B'_{j+1}$ intersected $B'_j$, it would be contained in $B(\mu_j,3R)\subseteq S_{i_{j+1}-1}$ and hence have capped probability at most $1/2$, contrary to its selection. By Claim~\ref{Claim:ballweight},
  \[
\lambda_{i_j}(B'_{j+1})\le0.4,
\qquad
\lambda_{i_{j+1}}(B'_{j+1})\ge0.6.
\]
Since all weights are nonincreasing, writing $W_i=w_i(H')$ gives
\[
\log\frac{W_{i_j}}{W_{i_{j+1}}}
\ge\log\frac{w_{i_{j+1}}(B'_{j+1})}{w_{i_j}(B'_{j+1})}
       \frac{W_{i_j}}{W_{i_{j+1}}}
=\log\frac{\lambda_{i_{j+1}}(B'_{j+1})}{\lambda_{i_j}(B'_{j+1})}
\ge\log\frac{0.6}{0.4}\ge\frac{1}{5}.
\]
The set $S_i$ is fixed from $i_j$ through $i_{j+1}-1$. If $\mathcal{E}_{p_j}=\sum_{l=i_j}^{i_{j+1}-1}e_l$, telescoping the definition of $\Delta_l$ yields
\begin{align*}
\sum_{l=i_j}^{i_{j+1}-1}\Delta_l
&=\log\frac{\lambda_{i_{j+1}}(h^*)}{\lambda_{i_j}(h^*)}
 +\log\frac{\lambda_{i_{j+1},H'\setminus S_{i_j}}(h^*)}{\lambda_{i_j,H'\setminus S_{i_j}}(h^*)}\\
&=-2\alpha\mathcal{E}_{p_j}
  +\log\frac{W_{i_j}}{W_{i_{j+1}}}
  +\log\frac{w_{i_j}(H'\setminus S_{i_j})}{w_{i_{j+1}}(H'\setminus S_{i_j})}\\
&\ge-2\alpha\mathcal{E}_{p_j}+\frac{1}{5}.
\end{align*}
The last logarithm is nonnegative because weights only decrease. There are at least $\mathcal{N}_k-1$ such completed phases before iteration $k$. In the initial and final incomplete phases, each increment satisfies $\Delta_i\ge-2\alpha e_i$. Summing, and using that $\Delta_i=e_i=0$ whenever $h^*\in S_i$, gives
\[
\psi_k\ge-2\alpha\mathcal{E}_k+\frac{1}{5}(\mathcal{N}_k-1),
\]
  as required.
\end{proof}
We next bound $\mathcal{E}_k$ in terms of $\psi_k$.
\begin{lemma}\label{LemmaErrorbound}
  With probability at least $1-\delta$, simultaneously for every $k\ge1$,
\[
\mathcal{E}_k\le\frac{1}{\alpha}\bra{\psi_k+\log\frac{1}{\delta}}.
\]
\end{lemma}
\begin{proof}
Let $d_i=\E[\Delta_i\mid\mathcal{F}_i]$ and $p_i=\E[e_i\mid\mathcal{F}_i]$. When $h^*\notin S_i$, $\err(h^*)\le3\eta$ implies
\[
p_i\le3\eta\max_x\frac{q_i(x)}{\D_X(x)}.
\]
 Lemma~\ref{lem:Deltai} therefore gives $\alpha p_i\le d_i/2$. This also holds when $h^*\in S_i$, since both sides then vanish. For $X_i=\Delta_i-\alpha e_i$, it follows that
\[
\E[X_i\mid\mathcal{F}_i]=d_i-\alpha p_i\ge\frac{1}{2}d_i\ge0,
\qquad |X_i|\le3\alpha.
\]
  By the conditional Cauchy--Schwarz inequality and Lemma~\ref{lem:Deltai},
\begin{align*}
\var[X_i\mid\mathcal{F}_i]
&\le\bra{\sqrt{\var[\Delta_i\mid\mathcal{F}_i]}
       +\alpha\sqrt{\var[e_i\mid\mathcal{F}_i]}}^2\\
&\le\bra{\sqrt{8\alpha d_i}+\sqrt{\alpha d_i/2}}^2
=\frac{25}{2}\alpha d_i
\le27\alpha\E[X_i\mid\mathcal{F}_i].
\end{align*}
  Apply Theorem~\ref{thm:freedman} with $b=3\alpha$ and $c_9=27\alpha$. For $0<\alpha\le1/100$, its coefficient $c$ is positive. With probability at least $1-\delta$, simultaneously for every $k$,
\[
\psi_k-\alpha\mathcal{E}_k
=\sum_{1\le i<k}X_i
\ge c\sum_{1\le i<k}\E[X_i\mid\mathcal{F}_i]-\log\frac{1}{\delta}
\ge-\log\frac{1}{\delta}.
\]
  Rearranging proves the lemma.
\end{proof}
Combining Lemmas~\ref{Lemma:GoodhExist} and~\ref{LemmaErrorbound} bounds the index of the first ball containing $h^*$.
\begin{lemma}\label{Lemma:LogBall}
For $k=1+O\bra{m^*\log\frac{|H'|}{\delta}}$ as in Lemma~\ref{Lemma:Concentration}, with probability at least $1-2\delta$, $h^*\in S_k$ and the ball that first contains $h^*$ is among the first $O\bra{\log\frac{|H'|}{\delta}}$ balls added. In particular, it is among the first $L=\lceil c_7\log(|H'|/\delta)\rceil$ balls for a sufficiently large absolute constant $c_7$.
 \end{lemma}
\begin{proof}
By a union bound, with probability at least $1-2\delta$, the events in Lemmas~\ref{Lemma:Concentration} and~\ref{LemmaErrorbound} both hold. Let $T\le k$ be the first iteration with $h^*\in S_T$. If $T=1$, then $\psi_T=0$. Otherwise, the potential inequality at iteration $T-1$ gives
\[
0\ge\phi_{T-1}\ge-2\log|H'|+\psi_{T-1}.
\]
Since $|\Delta_{T-1}|\le2\alpha$ and $\Delta_i=0$ for $i\ge T$,
\[
\psi_k=\psi_T\le2\log|H'|+2\alpha.
\]
Consequently, Lemmas~\ref{Lemma:GoodhExist} and~\ref{LemmaErrorbound} imply
\[
\mathcal{N}_k+1\le15\psi_k+10\log\frac{1}{\delta}+2
\le30\log|H'|+10\log\frac{1}{\delta}+30\alpha+2.
\]
This is the required bound on the index of the first ball containing $h^*$. The case $|H'|=1$ is immediate, since the first ball added to $S_i$ contains $h^*$.
\end{proof}

\subsection{Putting Everything Together}
We proved that after $O\bra{m^*\log\frac{|H'|}{\delta}}$ queries, one of the first $L$ balls contains $h^*$ with high probability. Let $C_L$ denote the first $\min\{|C|,L\}$ centres in the ordered list $C$. Hence, running stage two on $C_L$ does not take many more queries. We are ready to prove Theorem~ \ref{thm:main}.

\mainthm*
\begin{proof}
Let's pick $c_1,c_4,c_5$ as in Theorem~\ref{thm:upperbeta} and use confidence parameter $\delta/3$ in each stage. By Lemma~\ref{Lemma:LogBall}, with probability at least $1-2\delta/3$, a centre in $C_L$ is within $3c_4\eta+3c_5\eps$ of $h^*$. Since $\err(h^*)\le3\eta$, the best hypothesis in $C_L$ has error at most
\[
\eta_C=(3+3c_4)\eta+3c_5\eps.
\]
Apply Theorem~\ref{thm:slowbound} to $C_L$ with this known error upper bound. With probability at least $1-\delta/3$, stage two returns a hypothesis of error at most
\[
4\eta_C=(12+12c_4)\eta+12c_5\eps\le\eta+\eps,
\]
where the last inequality follows from $c_5=1/15$, $\eps\ge c_1\eta$, $c_1\ge90c_4$, and $c_4\ge300$. A union bound proves correctness.

The stage one algorithm makes
  \begin{align*}
&O\bra{m^*\bra{H,\D_X,c_4\eta,c_5\eps-(c_4+2)\eta,\frac{1}{100}}\log\frac{|H'|}{\delta}}\\
&\qquad\le O\bra{m^*\bra{H,\D_X,c_4\eta,\frac{c_5}{2}\eps,\frac{1}{100}}\log\frac{|H'|}{\delta}}.
\end{align*}
queries, since $c_5\eps-(c_4+2)\eta\ge(c_5/2)\eps$ under the stated conditions. Write $\overline m=m^*\bra{H,\D_X,c_4\eta,\frac{c_5}{2}\eps,\frac{1}{100}}$. The stage two algorithm makes $O\bra{|C_L|\log\frac{|C_L|}{\delta}}$ queries by Theorem~\ref{thm:slowbound}. The centres in $C_L$ have pairwise distances greater than $2R$, where $R=c_4\eta+c_5\eps$: a newly selected inner ball has positive mass outside every earlier radius-$3R$ ball, so the triangle inequality gives this separation. Since
\[
2\bra{c_4\eta+\frac{c_5}{2}\eps}\le2R,
\]
an algorithm solving the instance defining $\overline m$ must distinguish the realizable targets in $C_L$. The binary decision-tree counting bound gives $\log|C_L|\lesssim\overline m$. Together with $|C_L|\le L=O\bra{\log\frac{|H'|}{\delta}}$, stage two therefore takes
\[
O\bra{\bra{\overline m+\log\frac{1}{\delta}}\log\frac{|H'|}{\delta}}
\]
queries. Finally, $|H'|\le N(H,\D_X,\eta)$. Picking $c_2=c_4$ and $c_3=c_5/2$ gives the desired sample complexity.

  Computing the initial packing and all pairwise hypothesis distances takes $O(|H|^2|\Xcal|)$ time. Using these distances, checking the ball weights takes $O(|H|^2)$ time per round; computing $r$ takes $O(|H||\Xcal|)$ time and solving the optimization problem takes $O(|\Xcal|)$ time. Stage two on $C_L$ also takes polynomial time. Therefore, the overall running time is polynomial in the size of the problem.
\end{proof}

Similarly, we can prove Theorem~\ref{thm:upperbeta}.

\thmupperbeta*

\begin{proof}
  Apply the proof of Lemma~\ref{Lemma:LogBall} with $m^*$ replaced by $1/\beta$ and confidence parameter $\delta/3$; the assumed lower bound on the optimization objective supplies the potential drift directly. After $O\bra{\frac{1}{\beta}\log\frac{N}{\delta}}$ queries, with probability at least $1-2\delta/3$, a centre in $C_L$ is within $3c_4\eta+3c_5\eps$ of $h^*$, and $|C_L|=O\bra{\log\frac{N}{\delta}}$. The best hypothesis in $C_L$ thus has error at most $\eta_C=(3+3c_4)\eta+3c_5\eps$. Applying Theorem~\ref{thm:slowbound} with this error upper bound and confidence parameter $\delta/3$, stage two returns, with probability at least $1-\delta/3$, a hypothesis $\hat h$ satisfying
\[
\err(\hat h)\le4\eta_C=(12+12c_4)\eta+12c_5\eps\le\eta+\eps.
\]
Stage two makes $O\bra{\log\frac{N}{\delta}\log\frac{\log(N/\delta)}{\delta}}$ queries. Overall, the algorithm makes
  \[
O\bra{\frac{1}{\beta}\log\frac{N}{\delta}+\log\frac{N}{\delta}\log\frac{\log(N/\delta)}{\delta}}
\]
 queries and succeeds with probability at least $1-\delta$.
\end{proof}

\section{Query Complexity Lower Bound}
In this section we derive a lower bound for the agnostic binary classification problem, which we denote by \textsc{AgnosticLearning}. The lower bound is obtained from a reduction from minimum set cover, which we denote by \textsc{SetCover}. The problem \textsc{SetCover} consists a pair $(U,\mathcal{S})$, where $U$ is a ground set and $\mathcal{S}$ is a collection of subsets of $U$. The goal is to find a set cover $C\subseteq\mathcal{S}$ such that $\bigcup_{S\in C}S=U$ of minimal size $|C|$. We use $K$ to denote the cardinality of the minimum set cover.
\begin{lemma}[\cite{dinur2014analytical}, Corollary 1.5]\label{Lemma:LowerBoundExistence}
There exists hard instances \textsc{SetCoverHard} with the property $K\ge\log|U|$ such that for every $\gamma>0$, it is
NP-hard to approximate \textsc{SetCoverHard} to within $(1-\gamma)\ln |U|$.
\end{lemma}
\begin{proof}
This lemma directly follows from \citet[Corollary 1.5]{dinur2014analytical}. In their proof, they constructed a hard instance of \textsc{SetCover} from \textsc{LabelCover}. The size of the minimum cover $K\ge|V_1|=Dn_1$ and $\log|U|=(D+1)\ln n_1\le K$. So the instance in their proof satisfies the desired property.
\end{proof}
Then we prove the following lemma by giving a ratio-preserving reduction from \textsc{SetCover} to \textsc{AgnosticLearning}.
\begin{lemma}\label{Lemma:LowerBoundReduction}
If there exists a deterministic $\alpha$-approximation algorithm for $\textsc{AgnosticLearning}\bra{H, \D_x, \frac{1}{3|\Xcal|}, \frac{1}{3|\Xcal|}, \frac{1}{4|H|}}$, there exists a deterministic $2\alpha$-approximation algorithm for \textsc{SetCoverHard}.
\end{lemma}
\begin{proof}
Given an instance of \textsc{SetCoverHard}, for each $s\in\mathcal{S}$, number the elements $u\in s$ in an arbitrary order; let $f(s, u)$ denote the index of $u$ in $s$'s list (and padding 0 to the left with the extra bit). We construct an instance of \textsc{AgnosticLearning} as the following:
\begin{enumerate}
    \item Let the domain $\mathcal{X}$ have three pieces: $U$, $V:= \{(s, j) \mid s \in \mathcal{S}, j \in [1 + \log |s|]\}$, and $D=\curly{1,\dots,\log|U|}$, an extra set of $\log |U|$ more coordinates.
    \item On this domain, we define the following set of hypotheses:
    \begin{enumerate}
        \item For $u\in U$, define $h_u$ which only evaluates $1$ on $u \in U$ and on $(s,j) \in V$ if $u\in s$ and the $j$'th bit of $\bra{2f(s,u)+1}$ is 1.
        \item For $d\in D$, define $h_d$ which only evaluates $1$ on $d$.
        \item Define $h_0$ which evaluates everything to $0$.
    \end{enumerate}
    \item Let $\D_X$ be uniform distribution over $\mathcal{X}$ and set $\eta=\frac{1}{3\abs{\Xcal}}$ and $\eps=\frac{1}{3\abs{\Xcal}}$. Set $\delta=\frac{1}{4|H|}$.
\end{enumerate}
Any two hypotheses satisfy $\norm{h_1-h_2}\ge\frac{1}{\abs{\Xcal}}>\eps=\eta$, so $\err(h^*)=0$. First we show that $m^*\bra{H, \D_x, \frac{1}{3|\Xcal|}, \frac{1}{3|\Xcal|}, \frac{1}{4|H|}}\le K+\log|U|$. Indeed there exists a deterministic algorithm using $K+\log|U|$ queries to identify any hypothesis with probability 1. Given a smallest set cover $C$, the algorithm first queries all $(s,0) \in V$ for $s\in C$. If $h^*=h_u$ for some $u$, then for the $s \in \mathcal{S}$ that covers $u$, $(s, 0)$ will evaluate to true. The identity of $u$ can then be read out by querying $(s, j)$ for all $j$. The other possibilities--$h_d$ for some $d$ or $0$---can be identified by evaluating on all of $D$ with $\log U$ queries. The total number of queries is then at most $K+\log|U|$ in all cases, so $m^* \leq K + \log |U| \leq 2 K$.

We now show how to reconstruct a good approximation to set cover from a good approximate query algorithm. We feed the query algorithm $y = 0$ on every query it makes, and let $C$ be the set of all $s$ for which it queries $(s, j)$ for some $j$. Also, every time the algorithm queries some $u \in U$, we add an arbitrary set containing $u$ to $C$. Then the size of $C$ is at most the number of queries. We claim that $C$ is a set cover: if $C$ does not cover some element $u$, then $h_u$ is zero on all queries made by the algorithm, so $h_u$ is indistinguishable from $h_0$ and the algortihm would fail on either input $h_0$ or $h_u$.
Thus if $A$ is a deterministic $\alpha$-approximation algorithm for \textsc{AgnosticLearning}, we will recover a set cover of size at most $\alpha m^*\le\alpha\bra{K+\log|U|}\le2\alpha K$, so this gives a deterministic $2\alpha$-approximation algorithm for \textsc{SetCoverHard}.
\end{proof}

Similar results also holds for randomized algorithms, we just need to be slightly careful about probabilities.
\begin{lemma}
If there exists a randomized algorithm for $\textsc{AgnosticLearning}\bra{H, \D_x, \frac{1}{3|\Xcal|}, \frac{1}{3|\Xcal|}, \frac{1}{4|H|}}$, there exists a randomized $2\alpha$-approximation algorithm for \textsc{SetCoverHard} with success probability at least $\frac{2}{3}$.
\end{lemma}

\begin{proof}
We use the same reduction as in Lemma~\ref{Lemma:LowerBoundReduction}. Let $A$ be an algorithm solves $\textsc{AgnosticLearning}\bra{H, \D_x, \frac{1}{3|\Xcal|}, \frac{1}{3|\Xcal|}, \frac{1}{4|H|}}$. To obtain a set cover using $A$, we keeping giving $A$ label $0$ and construct the set $C$ as before. Let $q_C$ be a distribution over the reconstructed set $C$. Assume that by contradiction with probability at least $\frac{1}{3}$, $C$ is not a set cover. Then, with probability at least $1/3$, there is some element $v$ such that both $h_v$ and $h_0$ are consistent on all queries the algorithm made; call such a query set ``ambiguous''.

Then what is the probability that the agnostic learning algorithm fails on the input distribution that chooses $h^*$ uniformly from $H$? Any given ambiguous query set is equally likely to come from any of the consistent hypotheses, so the algorithm's success probability on ambiguous query sets is at most $1/2$. The chance the query set is ambiguous is at least $\frac{2}{3 |H|}$: a $\frac{1}{3H}$ chance that the true $h^*$ is $h_0$ and the query set is ambiguous, and at least as much from the other hypotheses making it ambiguous. Thus the algorithm's fails to learn the true hypothesis with at least $\frac{1}{3 |H|}$ probability, contradicting the assumed $\frac{1}{4 |H|}$ failure probability.

Therefore, a set cover of size at most $2\alpha K$ can be recovered with probability at least $\frac{1}{3}$ using the agnostic learning approximation algorithm.
\end{proof}

The following theorem will then follow.
\thmlower*
\begin{proof}
Let's consider the instance of set cover constructed in Lemma~\ref{Lemma:LowerBoundReduction}. Let $c=0.1$ and note that $0.1\log|H|\le0.49\log\frac{|H|}{2}$. If there exists a polynomial time $0.49\log \frac{|H|}{2}$ approximation algorithm for the instance, then there exists a polynomial time $0.98\log \frac{|H|}{2}\le0.98\log|U|$ approximation algorithm for \textsc{SetCoverHard}, which is a contradiction to Lemma~\ref{Lemma:LowerBoundExistence}.
\end{proof}

\end{document}